\def\year{2022}\relax
\documentclass[letterpaper]{article} 
\usepackage{aaai22}  
\usepackage{times}  
\usepackage{helvet}  
\usepackage{courier}  
\usepackage[hyphens]{url}  
\usepackage{graphicx} 
\usepackage{natbib}  
\usepackage{caption} 
\DeclareCaptionStyle{ruled}{labelfont=normalfont,labelsep=colon,strut=off} 
\usepackage{algorithm}
\usepackage{algorithmic}
\usepackage[utf8]{inputenc} 
\usepackage[T1]{fontenc}    
\usepackage{url}            
\usepackage{booktabs}       
\usepackage{amsfonts}       
\usepackage{nicefrac}       
\usepackage{microtype}      
\usepackage{lipsum}
\usepackage{amsmath, amsthm, amssymb, amsfonts}
\usepackage{multirow}
\usepackage{subfigure}
\usepackage{color}
\usepackage{enumitem}
\usepackage{textcomp}
\usepackage{microtype}
\usepackage{booktabs} 
\usepackage{times}
\usepackage{soul}
\usepackage{url}
\usepackage{enumerate}
\usepackage{dsfont}
\usepackage{multicol}
\usepackage{xcolor}
\newtheorem{theorem}{Theorem}
\newtheorem{definition}{Definition}
\newtheorem{assumption}{Assumption}
\usepackage[switch]{lineno}  %

\usepackage{newfloat}
\usepackage{listings}
\floatstyle{ruled}
\newfloat{listing}{tb}{lst}{}
\floatname{listing}{Listing}
\title{Constraints Penalized Q-learning for Safe Offline Reinforcement Learning}
\author {
    Haoran Xu\textsuperscript{\rm 1,2,3},
    Xianyuan Zhan\textsuperscript{\rm 4}\thanks{Corresponding author},
    Xiangyu Zhu\textsuperscript{\rm 2,3}
}
\affiliations {
    \textsuperscript{\rm 1} School of Computer Science and Technology, Xidian University, Xi’an, China\\
    \textsuperscript{\rm 2} JD iCity, JD Technology, Beijing, China\\
    \textsuperscript{\rm 3} JD Intelligent Cities Research\\
    \textsuperscript{\rm 4} Institute for AI Industry Research (AIR), Tsinghua University, Beijing, China\\
    \{ryanxhr, zhanxianyuan, zackxiangyu\}@gmail.com
}

\begin{document}

\maketitle

\begin{abstract}
We study the problem of safe offline reinforcement learning (RL), the goal is to learn a policy that maximizes long-term reward while satisfying safety constraints given only offline data, without further interaction with the environment. This problem is more appealing for real world RL applications, in which data collection is costly or dangerous.
Enforcing constraint satisfaction is non-trivial, especially in offline settings, as there is a potential large discrepancy between the policy distribution and the data distribution, causing errors in estimating the value of safety constraints.
We show that na\"ive approaches that combine techniques from safe RL and offline RL can only learn sub-optimal solutions. We thus develop a simple yet effective algorithm, Constraints Penalized Q-Learning (CPQ), to solve the problem. Our method admits the use of data generated by mixed behavior policies.
We present a theoretical analysis and demonstrate empirically that our approach can learn robustly across a variety of benchmark control tasks, outperforming several baselines.
\end{abstract}

\section{Introduction}

Reinforcement Learning (RL) has achieved great success in solving complex tasks, including games \cite{mnih2013playing,silver2017mastering}, and robotics \cite{levine2016end}.
However, most RL algorithms learn good policies only after millions of trials and errors in simulation environments.
Consider real-world scenarios (e.g. self-driving cars, industrial control systems), where we only have a batch of pre-collected data (non-optimal), including some unsafe attempts (e.g. high-speed collisions in self-driving cars), no further active online data collection is allowed. The question then arises: how can we derive an effective policy from these offline data while satisfying safety constraints?

Safe RL is usually modeled as a Constrained Markov Decision Process (CMDP) \cite{altman1999constrained}. There are typically two kinds of constraints: hard constraints and soft constraints. Hard constraints require no constraints violation at each time step of the trajectory, while soft constraints require the policy to satisfy constraints in expectation throughout the whole trajectory. In this work, we focus on the soft constraints.
There is a branch of related work for safe RL, with the main focus on safe exploration \cite{chow2017risk,achiam2017constrained,tessler2018reward}. However, none of these algorithms are off-policy and cannot be used in offline settings. 
Although one recent study aims towards batch policy learning under constraints \cite{le2019batch}, this method assumes sufficient exploration of the data collection policy. This requirement usually does not hold in real-world scenarios, especially in high-dimensional continuous control tasks.

The key challenge is how to evaluate constraint violations accurately while maximizing reward effectively. Typically, this needs to roll out the policy in the environment and evaluate constraint values by on-policy samples \cite{tessler2018reward}. However, it is impossible in the offline setting, as we only have access to samples in the offline dataset. 
Evaluating constraint values from offline data is non-trivial, it will encounter serious issues when the evaluated policy lies outside of the dataset distribution. 
In value-based RL approaches, this may introduce errors in the Q-function backup and it is impossible to collect online data to correct such errors. The problem will be further exacerbated when the offline dataset is generated by multiple conflicting behavior policies, as the policy may be biased to be unsafe or sub-optimal.

One can use an extra cost critic (like the reward critic) to learn the constraint values, and use a divergence penalty to control the deviation between the learned policy and the dataset distribution. However, we show that this na\"ive method is too conservative and will lead to a sub-optimal solution.
Our primary contribution is presenting a new algorithm, Constraints Penalized Q-Learning (CPQ), to address the above challenges.
The intuition is that besides those original unsafe actions, we additionally make those actions that are out of the data distribution unsafe. To accomplish this, we modify the Bellman update of reward critic to penalize those state action pairs that are unsafe. 
CPQ does not use an explicit policy constraint and will not be restricted by the density of the dataset distribution, it admits the use of datasets generated by mixed behavior policies.
We also provide a theoretical error bound analysis of CPQ under mild assumptions. Through systematic experiments, we show that our algorithm can learn robustly to maximize rewards while successfully satisfying safety constraints, outperform all baselines in benchmark continuous control tasks.

\section{Related Work}

\subsection{Safe Reinforcement Learning}
Safe RL can be defined as the process of learning policies that maximizes long-term rewards while ensuring safety constraints.
When Markov transition probability is known, a straightforward approach is based on linear programming \cite{altman1999constrained}. In model-free settings, Lagrangian-based methods \cite{chow2017risk,tessler2018reward} augment
the standard expected reward objective with a penalty of constraint violation and solve the resulting problem with a learnable Lagrangian multiplier. However, Lagrangian-based policy can only asymptotically satisfy the constraint and makes no safety guarantee during the training process when interaction with the real-world environment is required\footnote{This property does not impact offline RL settings, as the training process does not involve online environment interaction.}.
Constrained policy optimization (CPO) \cite{achiam2017constrained} extends trust-region optimization \cite{schulman2015trust}, which can satisfy constraints during training, but the computational expense dramatically increases with multiple constraints.
There are some other approaches designed for convex constraints \cite{miryoosefi2019reinforcement} or hard constraints \cite{dalal2018safe,satija2020constrained}. However, all of these algorithms are on-policy, thus cannot be applied to the offline setting.
Constrained Batch Policy Learning (CBPL) \cite{le2019batch} considers safe policy learning offline, it uses Fitted Q Evaluation (FQE) to evaluate the safe constraints and learn the policy by Fitted Q Iteration (FQI), through a game-theoretic framework.

\subsection{Offline Reinforcement Learning}
Offline RL (also known as batch RL \cite{lange2012batch} or fully off-policy RL) considers the problem of learning policies from offline data without interaction with the environment.
One major challenge of offline RL is the \textbf{distributional shift} problem \cite{levine2020offline}, which incurs when the policy distribution deviates largely from the data distribution. 
Although off-policy RL methods \cite{mnih2013playing,lillicrap2016continuous} are naturally designed for tackling this problem, they typically fail to learn solely from fixed offline data and often require a growing batch of online samples for good performance.
Most recent methods attempted to solve this problem by constraining the learned policy to be “close” to the behavior policy. BCQ \cite{fujimoto2019off} learns a generative model for the behavior policy and adds small perturbations to it to stay close to the data distribution while maximizing the reward. Some other works use divergence penalties (such as KL divergence in BRAC \cite{wu2019behavior} or maximum mean discrepancy (MMD) in BEAR \cite{kumar2019stabilizing}) instead of perturbing actions. CQL \cite{kumar2020conservative} uses an implicit Q-value constraint between the learned policy and dataset samples, which avoids estimating the behavior policy.
The distributional shift problem can also be solved by model-based RL through a pessimistic MDP framework \cite{yu2020mopo,kidambi2020morel,zhan2021deepthermal} or by constrained offline model-based control \cite{argenson2021modelbased, zhan2021model}.

\section{Preliminary}
\subsection{Background}
A Constrained Markov Decision Process (CMDP) is represented by a tuple $(\mathcal{S}, \mathcal{A}, r, c, P, \gamma, \eta)$, where $\mathcal{S} \subset \mathbb{R}^{n}$ is the closed and bounded state space and $\mathcal{A} \subset \mathbb{R}^{m}$ is the action space. Let $r: \mathcal{S} \times \mathcal{A} \mapsto [0,\overline{R}]$ and $c: \mathcal{S} \times \mathcal{A} \mapsto [0,\overline{C}]$ denote the reward and cost function, bounded by $\overline{R}$ and $\overline{C}$. Let $P: \mathcal{S} \times \mathcal{A} \times \mathcal{S} \mapsto[0,1]$ denote the (unknown) transition probability function that maps state-action pairs to a distribution over the next state. Let $\eta$ denote the initial state distribution. And finally, let $\gamma \in [0,1)$ denote the discount factor for future reward and cost. 
A policy $\pi:\mathcal{S} \mapsto \mathcal{P}(\mathcal{A})$ correspounds to a map from states to a probability distribution over actions. Specifically, $\pi(a|s)$ denotes the probability of taking action $a$ in state $s$. In this work, we consider parametrized policies (e.g. neural networks), we may use $\pi_{\theta}$ to emphasize its dependence on parameter $\theta$.
The cumulative reward under policy $\pi$ is denoted as $R(\pi) = \mathbb{E}_{\tau \sim \pi}[\sum_{t=0}^{\infty} \gamma^{t} r(s_t, a_t)]$, where $\tau=(s_0, a_0, s_1, a_1,...)$ is a trajectory and $\tau \sim \pi$ means the distribution over trajectories is induced by policy $\pi$.
Similarly, the cumulative cost takes the form as $C(\pi) = \mathbb{E}_{\tau \sim \pi}[\sum_{t=0}^{\infty} \gamma^{t} c(s_t, a_t]]$.

Off-policy RL algorithms based on dynamic programming maintain a parametric Q-function $Q_{\phi}(s, a)$. 
Q-learning methods \cite{watkins1992q} train the Q-function by iteratively applying the Bellman optimality operator $\mathcal{T}^{*}Q(s, a) := r(s, a)+\gamma \mathbb{E}_{s'}[\max_{a'}Q(s',a')]$. 
In an actor-critic algorithm, the Q-function (critic) is trained by iterating the Bellman evaluation operator $\mathcal{T}^{\pi} Q=r+\gamma P^{\pi} Q$, where $P^{\pi}$ is the transition matrix coupled with the policy: $P^{\pi} Q(s,a)=\mathbb{E}_{s' \sim T\left(s' |s, a\right), a' \sim \pi\left(a'|s'\right)}\left[Q\left(s', a' \right) \right]$, and a separate policy is trained to maximize the expected Q-value.
Since the replay buffer typically does not contain all possible transitions $\left(s, a, s'\right)$, the policy evaluation step actually uses an empirical Bellman operator that only backs up a single sample $s'$. 
Notice that $\pi$ is trained to maximize Q-values, it may be biased towards out-of-distribution (OOD) actions with erroneously high Q-values. In standard (online) RL, such errors can be corrected by interacting with the environment and observing its actual value.

In our problem, we assume no interaction with the environment and only have a batch, offline dataset $\mathcal{B}$ = ${(s,a,s',r(s,a),c(s, a))}$, generated by following unknown arbitrary behavior policies. Note that these behavior policies may generate trajectories that violate safety constraints.
We use $\pi_\beta$ to represent the empirical behavior policy of the dataset, formally, $\pi_\beta(a_0|s_0):=\frac{\sum_{\mathbf{s}, \mathbf{a} \in \mathcal{B}} \mathbf{1}[s=s_0,a=a_0]}{\sum_{s \in \mathcal{B}} \mathbf{1}[s=s_0]}$, for all state $s_0 \in \mathcal{B}$. We use $\mu^{\beta}(s)$ to represent the discounted marginal state-distribution of $\pi_\beta(a|s)$, thus the dataset $\mathcal{B}$ is sampled from $\mu^{\beta}(s) \pi_\beta(a|s)$. 
The goal of safe offline learning is to learn a policy $\pi$ from $\mathcal{B}$ that maximizes the cumulative reward while satisfying the cumulative cost constraint, denoted as 
\begin{equation*} 
\label{eq:1}
\begin{split}
& \max_{\pi} \quad R(\pi) \\
& \mathrm{s.t.} \quad C(\pi) \leq l
\end{split}
\end{equation*}
where $l$ is the safe constraint limit (a known constant).

\subsection{An Na\"ive Approach}
An na\"ive approach to solve safe offline RL is combining techniques from safe RL and offline RL. For example, we could train an additional cost critic Q-network (to get the value of cumulative cost like in \cite{liang2018accelerated,ha2020learning}), along with a divergence constraint to prevent large distributional shift from $\pi_\beta$. Formally, we update both the reward and cost critic Q-network by the empirical Bellman evaluation operator $\mathcal{T}^{\pi}$ for $\left(s, a, s^{\prime}, r, c \right) \sim \mathcal{B}$:
\begin{align*}
Q_r(s, a)=r+\gamma\mathbb{E}_{a'\sim\pi(\cdot|s')}\left[Q_r(s', a')\right] \\ 
Q_c(s, a)=c+\gamma\mathbb{E}_{a'\sim\pi(\cdot|s')}\left[Q_c(s', a')\right]
\end{align*}
and then the policy can be derived by solving the following optimization problem:
\begin{align} 
\label{eq_opt}
\pi_\theta := \max_{\pi \in \Delta_{|S|}} \mathbb{E}_{s\sim \mathcal{B}, a\sim\pi(\cdot|s)} \left[ Q_r(s, a) \right] \notag \\
\mathrm{s.t.} \quad \mathbb{E}_{s\sim \mathcal{B}, a \sim \pi(\cdot|s)} \left[ Q_c(s, a) \right] \leq l \tag{Constraint 1} \\
D \left(\pi, \pi_{\beta}\right) \leq \xi \tag{Constraint 2}
\end{align}
$D$ can be any off-the-shelf divergence metrics (e.g., KL divergence or MMD distance) and $\xi$ is an approximately chosen small value. We can convert the constrained optimization problem to an unconstrained form by using the Lagrangian relaxation procedure, and solve it by dual gradient descent.

However, we argue that in this approach, Constraint 1 and Constraint 2 may not be satisfied simultaneously. Suppose $\mathcal{B}$ consists of transitions from both safe and unsafe policies. When the policy $\pi$ satisfies Constraint 2, it will match the density of the behavior policy distribution, when the behavior policy distribution contains some unsafe actions, the resulting policy may violate Constraint 1. One may consider to subtract transitions of the safe policy from $\mathcal{B}$ to construct a new "safe dataset", and only use it for training. Although in this case, Constraint 1 and Constraint 2 can be both satisfied, however, the missing of high-reward transitions will make the resulting policy sub-optimal. In principle, by carefully "stitching" together transitions from both safe and unsafe policies, the policies ought to produce trajectories with maximized cumulative reward while still satisfying safety constraints.

\section{Constraints Penalized Q-Learning}
In this section, we introduce our method, Constraints Penalized Q-Learning (CPQ), a simple yet effective algorithm for safe offline RL.
The key idea is to make OOD actions "unsafe" and update the reward critic using only state-action pairs that are "safe".
CPQ avoids explicit policy constraints, it involves the following three steps:

\noindent \textbf{Step 1:} \ We first make Qc-values of OOD actions larger than the safe constraint limit, 
we accomplish this by adding an additional term to the original objective of Bellman evaluation error, yielding a new objective:
\begin{equation}
\label{eq_qc}
\min_{Q_c} \mathbb{E}_{s,a,s'\sim  \mathcal{B}} \left[ \left( Q_c - \mathcal{T}^{\pi} Q_c \right)^2 \right] 
- \alpha \mathbb{E}_{s\sim \mathcal{B}, a\sim \nu }\left[Q_c(s, a)\right]
\end{equation}
Aside from the standard Bellman evaluation error term, Equation (\ref{eq_qc}) also maximizes the Qc-values at all states in the dataset $\mathcal{B}$, for those actions induced from the distribution $\nu$. Intuitively, if we choose $\nu$ to be a distribution that generates OOD actions, those OOD actions' Qc-values will be pushed up. Note that the Qc-values of in-distribution actions would be pushed down to obey the Bellman backup by the standard Bellman error term. Therefore, with an appropriate weight $\alpha$, we will only overestimate Qc-values of OOD actions, while keep them unchanged for in-distribution actions.


The remaining question is how to get the distribution $\nu$ that generates OOD actions. We avoid this hard problem by performing the OOD detection \cite{ren2019likelihood,liu2020energy}. As the policy $\pi$ is trained to maximize the reward critic, we only need to ensure that actions sampled by $\pi$ are not OOD. To do so, we pretrain the Conditional Variational Autoencoder (CVAE) to model the behavior policy of the dataset and utilize the latent space to do OOD detection. More specifically, we train the state-conditional VAE based on the following evidence lower bound (ELBO) objective on the log-likelihood of the dataset:
\begin{equation}
\label{eq_vae}
\max_{\omega_1, \omega_2} \mathbb{E}_{z \sim q_{\omega_2}}\left[\log p_{\omega_1}(a|s,z)\right] - \beta D_{\text{KL}}\left[q_{\omega_2}(z|s,a) \| p_{\omega_1}(z)\right]
\end{equation}
The first term represents the reconstruction loss and the second term is the KL-divergence between the encoder output and the prior of $z$. 
Note that action $a$ at state $s$ will have a high probability under the behavior data distribution if the value of $z \sim q_{\omega_2}(s,a) $ has a high probability under the prior $p(z)$. Since $p(z)$ is set to be $\mathcal{N}(0,1)$, we can let $\nu(s)=a$ if 
$D_{\text{KL}} [q_{\omega_2}(z|s,a) || \mathcal{N}(0,1) ] \geq d$, by introducing a hyperparameter $d$ to control the threshold. 
Previous works \cite{fujimoto2019off,kumar2019stabilizing} also use CVAE, but they use it to sample actions and compute the value of divergence metrics, which is different from our usage.

\noindent \textbf{Step 2:} \ In Step 1, the cost critic learned by CPQ is somewhat "distorted", i.e., Qc-values of OOD actions are likely to be larger than their true values and extrapolate to actions near the boundary of in-distribution actions.
In preliminary experiments, we found it did not work well when using the distorted cost critic to update the policy by dual gradient descent (i.e., $\max_{\pi} Q^{\pi}_{r}-\lambda Q^{\pi}_{c}$).
Fortunately, Qr-values in Step 1 remain unchanged, so we can update the policy by only maximizing Qr-values. We modify the reward critic's Bellman update to only backup from state action pairs that are both \textbf{constraint safe} and \textbf{in-distribution safe}, this is accomplished by multiplying $Q_r(s',a')$ by an indicator.
We define the empirical \textbf{Constraints Penalized Bellman operator} $\mathcal{T}_{P}^{\pi}$ for $\left(s, a, s^{\prime}, r, c \right) \sim \mathcal{B}$, as 
\begin{equation*}
\begin{split}
\label{eq_qr} 
\mathcal{T}_{P}^{\pi} Q_r(s, a) = r+\gamma \mathbb{E}_{a'\sim \pi} \left[\mathds{1}\left( Q_c(s',a') \leq l \right) Q_r(s',a') \right]
\end{split}
\end{equation*}
where $\mathds{1}$ is the indicator function.
It can be shown that $\mathcal{T}_{P}^{\pi}$ reduces the update from those unsafe state action pairs by using a pessimistic estimate of 0 to those pairs.
Given the offline dataset $\mathcal{B}$, we update the reward critic by minimizing the mean-square error (MSE) as 
\begin{equation}
\label{eq_qr_loss}
\min_{Q_r} \mathbb{E}_{s,a,s'\sim \mathcal{B}} \left[ \left( Q_r(s, a)-\mathcal{T}_{P}^{\pi} Q_r(s, a) \right)^2 \right]
\end{equation}

\noindent \textbf{Step 3:} \ Finally, in the policy improvement step, to ensure the final policy is safe, CPQ applies the indicator to the computed state-action values before performing maximization:
\begin{equation}
\label{eq_pi_loss}
\pi_\theta := \max_{\pi \in \Delta_{|S|}} \mathbb{E}_{s\sim \mathcal{B}}\mathbb{E}_{a\sim\pi(\cdot|s)} \left[ \mathds{1}\left( Q_c(s, a) \leq l \right) Q_r(s, a) \right]
\end{equation}

\noindent \textbf{Connections to CQL} \
CQL adds two penalty terms to the standard Bellman error term, the first term is minimizing Qr-values of actions from the learned policy, the second term is maximizing Qr-values of actions from the dataset. CQL makes the Qr-value landscape to have higher values in the area of data distribution than in the area of the policy distribution, since CQL only considers reward maximization, it gets rid of the bad impact of OOD actions. 
In our problem, the policy is trained to both maximize the reward and satisfy safety constraints. We cannot simply follow CQL by maximizing Qc-values of actions from the policy and minimizing Qc-values of actions from the dataset. Maximizing Qc-values of actions from the policy will deteriorate the performance when the policy outputs in-distribution actions. Hence, we detect the actions output by the policy and only make Qc-values of those OOD actions large.

\noindent \textbf{Practical Considerations} \
To reduce the number of hyperparameters, we can automatically tune $\alpha$ as shown below: 
\begin{equation}
\label{eq_qc_lag}
\begin{split}
\min_{Q_c} \max_{\alpha \geq 0} \mathbb{E}_{s,a,s'\sim  \mathcal{B}} \left[ \left( Q_c - \mathcal{T}^{\pi} Q_c \right)^2 \right] \\
- \alpha \left( \mathbb{E}_{s\sim \mathcal{B}, a\sim \nu}\left[Q_c(s, a)\right] - l_c \right)
\end{split}
\end{equation}
Equation (\ref{eq_qc_lag}) implies that $\alpha$ will stop increasing if Qc-values of OOD actions are larger than $l_c$. The parameter $l_c$ should be chosen to be larger than the constraint threshold $l$, in practice, we use $l_c =1.5 \times l$ across all tasks.
We use $\beta$-VAE \cite{Higgins2017betaVAELB} to learn better disentangled latent space representations compared to the original VAE framework (can be seen as a special case of $\beta$-VAE with $\beta$ = 1).
We sample $n$ actions from the policy and choose $\nu$ to be all the actions that violate the latent space threshold $d$. If none of the $n$ actions violates, Equation (\ref{eq_qc}) is reduced to the original Bellman evaluation error objective.
We also adopt the double-Q technique \cite{fujimoto2018addressing} to penalize the uncertainty in value estimations, we select the minimal value of two reward critics when computing the target Qr-values. This trick is not applied to the cost critic, as it will tend to underestimate the Qc-values.
Implementation details and hyperparameter choices can be found in Appendix B.
The pseudo-code of CPQ is presented in Algorithm 1.

\begin{algorithm}
\small
\caption{Constraints Penalized Q-Learning (CPQ)}
\label{alg1}
\begin{algorithmic}[1]
\REQUIRE 
$\mathcal{B}$, constraint limit $l$, threshold $d$.
\STATE Initialize encoder $E_{\omega_1}$ and decoder $D_{\omega_2}$.
\STATE \textbf{// VAE Training}
\FOR{$t=0,1,...,M$}
\STATE Sample mini-batch of state-action pairs $(s,a) \sim \mathcal{B}$
\STATE Update encoder and decoder by Eq.(\ref{eq_vae})
\ENDFOR
\STATE \textbf{// Policy Training}
\STATE Initialize reward critic ensemble $\{Q_{r_i}(s,a|\phi_{r_i})\}_{i=1}^{2}$ and cost critic $Q_{c}(s,a|\phi_{c})$, actor $\pi_{\theta}$, Lagrange multiplier $\alpha$, target networks $\{Q'_{r_i}\}_{i=1}^{2}$ and $Q'_{c}$, with $\phi'_{r_i} \leftarrow \phi_{r_i}$ and $\phi'_{c} \leftarrow \phi_{c}$
\FOR{$t=0,1,...,N$}
\STATE Sample mini-batch of transitions $(s,a,r,c,s') \sim \mathcal{B}$
\STATE Sample $n$ actions $\{a_{i} \sim \pi_{\theta}(a|s) \}_{i=1}^{n}$, get latent mean and std $\{ \mu_{i}, \sigma_{i} = E_{\omega_1}(s, a_{i}) \}_{i=1}^{n}$ and extract $m$ ($m \geq 0$) actions $\{ a_{j} | D_{\text{KL}}(\mathcal{N}(\mu_{j}, \sigma_{j}) \| \mathcal{N}(0,1)) \geq d \}_{j=1}^{m}$ from them.
\STATE Let $Q_c(s, \nu(s))= \frac{1}{m} \sum_{j} Q_c(s, a_{j})$ if $m>0$ otherwise $0$.
\STATE Update cost critic by Eq.(\ref{eq_qc}) and reward critics by Eq.(\ref{eq_qr_loss}).
\STATE Update actor by Eq.(\ref{eq_pi_loss}) using policy gradient.
\STATE Update target cost critic: $\phi'_{c} \leftarrow \tau \phi_{c}+(1-\tau)\phi'_{c}$
\STATE Update target reward critics: $\phi'_{r_i} \leftarrow \tau \phi_{r_i}+(1-\tau)\phi'_{r_i}$
\ENDFOR
\end{algorithmic}
\end{algorithm}

\section{Analysis}
In this section, we give a theoretical analysis of CPQ, specifically, we proof that we can learn a safe and high-reward policy given only the offline dataset.
We first give the notation used for the proof and define what is \textit{Out-of-distribution Action Set}, then we proof that CPQ can make Qc-values of out-of-distribution actions greater than $l$ with specific $\alpha$. Finally we give the error bound of the difference between the Qr-value obtained by iterating \textit{Constraints Penalized Bellman operator} and the Qr-value of the optimal safe policy $\pi^{*}$ that can be learned from the offline dataset.

\noindent \textbf{Notation} \ 
Let $Q^{k}$ denotes the true tabular Q-function at iteration $k$ in the MDP, without any correction. 
In an iteration, the current tabular Q-function, $Q^{k+1}$ is related to the previous tabular Q-function iterate $Q^{k}$ as: $Q^{k+1}=\mathcal{T}^{\pi} Q^{k}$. 
Let $\hat{Q}^{k}$ denote the $k$-th Q-function iterate obtained from CPQ. Let $\hat{V}^{k}$ denote the value function as $\hat{V}^{k}:=\mathbb{E}_{\mathbf{a} \sim \pi(\mathbf{a} \mid \mathbf{s})}\left[\hat{Q}^{k}(\mathbf{s}, \mathbf{a})\right]$.

We begin with the definition of \textit{Out-of-distribution Action Set}:
\begin{definition} [Out-of-distribution Action Set]
Given a dataset $\mathcal{B}$, its empirical behavior policy $\pi_\beta$ and $\epsilon \in (0,1)$, we call a set of actions $A_{\epsilon}$ (generated by the policy $\nu$) as the out-of-distribution action set, if
$\forall{\mathbf{s}} \in \mathcal{B}, \forall{\mathbf{a}} \in A_{\epsilon}, \frac{\pi_{\beta}(\mathbf{a}|\mathbf{s})}{\nu(\mathbf{a}|\mathbf{s})} \leq \epsilon$. 
\end{definition}
Intuitively, for those out-of-distribution actions (i.e., unlike to be in the data distribution), $\nu(\mathbf{a}|\mathbf{s})$ will be large while $\pi_{\beta}(\mathbf{a}|\mathbf{s})$ will be small. In contrast to out-of-distribution actions, in-distribution actions refer to those actions $\mathbf{a} \sim \pi_{\beta}(\mathbf{a}|\mathbf{s})$, i.e., have good support in the data distribution. Notice that here we do not care about out-of-distribution states as states used for training are sampled from $\mathcal{B}$. After introducing the out-of-distribution action set, we now show that we can make Qc-values of $A_{\epsilon}$ greater than $l$ with appropriate $\alpha$ when updating the cost critic by Equation (\ref{eq_qc}).

\begin{theorem}
For any $\nu(\mathbf{a} | \mathbf{s})$ with $\operatorname{supp} \nu \subset \operatorname{supp} \pi_{\beta}$, $\forall \mathbf{s} \in \mathcal{B}, \mathbf{a} \in A_{\epsilon}$, $\hat{Q}_c^{\pi}$ (the $Q$-function obtained by iterating Equation (\ref{eq_qc})) satisfies:
\[
 \hat{Q}_{c}^{\pi}(\mathbf{s},\mathbf{a}) = Q_{c}^{\pi}(\mathbf{s},\mathbf{a}) + \frac{\alpha}{2} \cdot \left[\left(I-\gamma P^{\pi}\right)^{-1} \frac{\nu(\mathbf{s}
 |\mathbf{a})}{\pi_{\beta}(\mathbf{s}|\mathbf{a})}\right](\mathbf{s}, \mathbf{a})
\]
and we can get $\hat{Q}_c^{\pi}(\mathbf{s},\mathbf{a}) \geq l, \forall \mathbf{s} \in \mathcal{B}, \mathbf{a} \in A_{\epsilon}$ if we choose $\alpha \geq \max \{ 2\epsilon \max_{\mathbf{s}, \mathbf{a}} \left( l-Q_c^{\pi}(\mathbf{s}, \mathbf{a}) \right) \left(I-\gamma P^{\pi}\right) (\mathbf{s}, \mathbf{a}), 0 \}$.
\end{theorem}

\begin{proof}
By setting the derivative of Equation (\ref{eq_qc}) to 0, we obtain the following expression for $\hat{Q}_c^{k+1}$ in terms of $\hat{Q}_c^{k}$,
\begin{equation}
\label{eq_derivative}
\forall k, \quad \hat{Q}_c^{k+1}(\mathbf{s}, \mathbf{a})=\mathcal{T}^{\pi} \hat{Q}^{k}_c(\mathbf{s}, \mathbf{a})+ \frac{\alpha}{2} \cdot \frac{\nu(\mathbf{a}|\mathbf{s})}{\pi_{\beta}(\mathbf{a}|\mathbf{s})}
\end{equation}

Since $\nu(\mathbf{a}|\mathbf{s})>0, \alpha>0, \pi_{\beta}(\mathbf{a}|\mathbf{s})>0$, we observe that at each iteration we enlarge the next Qc-value, i.e. $\hat{Q}_c^{k+1} \geq \mathcal{T}^{\pi} \hat{Q}_c^{k}$.
Now let's examine the fixed point of Equation (\ref{eq_derivative}) as, 
\begin{equation*}
\begin{split}
&\hat{Q}_c^{\pi}(\mathbf{s}, \mathbf{a}) = \mathcal{T}^{\pi} \hat{Q}_c^{\pi}(\mathbf{s}, \mathbf{a})+\frac{\alpha}{2} \cdot \frac{\nu(\mathbf{a}|\mathbf{s})}{\pi_{\beta}(\mathbf{a}|\mathbf{s})} \\
&= c + \gamma P^{\pi} \hat{Q}_c^{\pi}(\mathbf{s}, \mathbf{a}) + \frac{\alpha}{2} \cdot \frac{\nu(\mathbf{a}|\mathbf{s})}{\pi_{\beta}(\mathbf{a}|\mathbf{s})} \\
&= Q_c^{\pi}(\mathbf{s}, \mathbf{a}) (I-\gamma P^{\pi}) + \gamma P^{\pi} \hat{Q}_c^{\pi}(\mathbf{s}, \mathbf{a}) + \frac{\alpha}{2} \cdot \frac{\nu(\mathbf{a}|\mathbf{s})}{\pi_{\beta}(\mathbf{a}|\mathbf{s})} \\
&= Q_c^{\pi}(\mathbf{s}, \mathbf{a}) + \gamma P^{\pi} \left[ \hat{Q}_c^{\pi}(\mathbf{s}, \mathbf{a}) - Q_c^{\pi}(\mathbf{s}, \mathbf{a}) \right] + \frac{\alpha}{2} \cdot \frac{\nu(\mathbf{a}|\mathbf{s})}{\pi_{\beta}(\mathbf{a}|\mathbf{s})}
\end{split}
\end{equation*}

So we can get the relationship between $\hat{Q}_c^{\pi}$ and the true Qc-value $Q_c^{\pi}$ as,
\begin{equation*}
\hat{Q}_c^{\pi}(\mathbf{s}, \mathbf{a}) = Q_c^{\pi}(\mathbf{s}, \mathbf{a}) + \frac{\alpha}{2} \left(I-\gamma P^{\pi}\right)^{-1} \left[ \frac{\nu(\mathbf{a}|\mathbf{s})}{\pi_{\beta}(\mathbf{a}|\mathbf{s})}\right](\mathbf{s}, \mathbf{a})
\end{equation*}

If $Q_c^{\pi}(\mathbf{s}, \mathbf{a}) \geq l$, i.e., the true Qc-value of this state-action pair is greater than the constraint threshold, we don't need to enlarge the Qc-value, set $\alpha=0$ works. 
Otherwise, if $Q_c^{\pi}(\mathbf{s}, \mathbf{a}) \leq l$, the choice of $\alpha$ that guarantee $\hat{Q}_c^{\pi}(\mathbf{s}, \mathbf{a}) \geq l$ for $\mathbf{a} \in A_{\epsilon}$, is then given by:
\begin{align}
&\alpha \geq 2 \left( l-Q_c^{\pi}(\mathbf{s}, \mathbf{a}) \right) \left(I-\gamma P^{\pi}\right) \left[ \frac{\hat{\pi}_{\beta}(\mathbf{a}|\mathbf{s})}{\mu(\mathbf{a}|\mathbf{s})}\right](\mathbf{s}, \mathbf{a}) \notag\\ \label{eq_7}
\Longrightarrow &\alpha \geq 2 \epsilon \cdot \max_{\mathbf{s}, \mathbf{a}} \left( l-Q_c^{\pi}(\mathbf{s}, \mathbf{a}) \right) \left(I-\gamma P^{\pi}\right) (\mathbf{s}, \mathbf{a})
\end{align}
Note that $\left(I-\gamma P^{\pi}\right)$ is a matrix (the inverse of the state occupancy matrix \cite{sutton1998introduction}) with all non-negative entries, and (\ref{eq_7}) holds because of the definition of out-of-distribution action set.
In all, choose $\alpha \geq \max \{2\epsilon \max_{\mathbf{s}, \mathbf{a}} \left( l-Q_c^{\pi}(\mathbf{s}, \mathbf{a}) \right) \left(I-\gamma P^{\pi}\right) (\mathbf{s}, \mathbf{a}), 0 \}$ will satisfy $\forall \mathbf{s} \in \mathcal{B}, \mathbf{a} \in A_{\epsilon}, \hat{Q}_c^{\pi}(\mathbf{s},\mathbf{a}) \geq l$.
\end{proof}

Now we show the error bound of the difference between the value obtained by CPQ and the value of the optimal safe policy $\pi^{*}$ on the dataset.
\begin{theorem}
Let $\| \hat{Q}_r^{k}-\mathcal{T}_{P}^{\pi} \hat{Q}_r^{k-1} \|_{\mu^{\beta}}$ be the squared approximation error of the Constraints Penalized Bellman operator $\mathcal{T}_{P}^{\pi}$ at iteration $k$. Let $\| Q_r^{k}-\mathcal{T}^{\pi} Q_r^{k-1} \|_{\mu^{\beta}}$ be the squared approximation error of the Bellman evaluation operator $\mathcal{T}^{\pi}$ at iteration $k$. If these two errors are bounded by $\delta$, then $\forall \mathbf{s} \in \mathcal{B}$, we have: 
\begin{align*}
1) \lim_{k \rightarrow \infty} \hat{V}_{c}^k \leq l \ ; \ 
2) \lim_{k \rightarrow \infty} \left|V_r^{*}-\hat{V}_r^k \right| \leq \frac{4\gamma}{(1-\gamma)^{3}} G(\epsilon) \sqrt{\delta}
\end{align*}
where $G(\epsilon)=\sqrt{(1-\gamma)/\gamma}+\sqrt{\epsilon/g(\epsilon)}$ and define $g(\epsilon) := \min _{\mu^{\pi}(s)>0}[\mu^{\beta}(s)]$, $g(\epsilon)$ captures the minimum discounted visitation probability of states under behaviour policy.
\end{theorem}

\begin{proof}
For 1), it can be easily derived from (\ref{eq_pi_loss}) that when $k \rightarrow \infty$, $\forall s \in \mathcal{B}$, we have $\hat{V}_c(\mathbf{s}) = \mathbb{E}_{\mathbf{a} \sim \pi(\cdot|\mathbf{s})} \left[ \hat{Q}_c(\mathbf{s},\mathbf{a}) \right] \leq l$.
For 2), we give a proof sketch here, detailed proof can be found in Appendix A. The proof sketch goes as follow, we first convert the performance difference between $\pi^{*}$ and $\pi_{t}$ to a value function gap that is filtered by the indicator $\mathds{1}\left( \hat{Q}_c(\mathbf{s'}
,\mathbf{a'}) \leq l \right)$ (for simplicity, we denote it as $P_c$ below):
\begin{equation*}
V_{r}^{*}-\hat{V}_{r}^{k} \leq \frac{1}{\gamma} \mathbb{E}_{\mathbf{s},\mathbf{a} \sim \nu} \left[ \left| P_c(\mathbf{s},\mathbf{a})\left(Q_{r}^{*}(\mathbf{s},\mathbf{a})-\hat{Q}_{r}^{k}(\mathbf{s},\mathbf{a})\right)\right| \right]
\end{equation*}
where $\nu$ is any distribution over state-action space $\mathcal{S} \times \mathcal{A}$.
We then prove $\left| P_c(Q_{r}^{*}-\hat{Q}_{r}^{k}) \right|_{\nu} = \mathbb{E}_{\nu} \left| P_c(Q_{r}^{*}-\hat{Q}_{r}^{k}) \right|$ can be bounded by $C \left( \left| \hat{Q}_{r}^{k}- \mathcal{T}_{P}^{\pi} \hat{Q}_{r}^{k-1} \right|_{\mu^{\beta}} + \left| Q_{r}^{*}-\mathcal{T}^{\pi} Q_{r}^{*} \right|_{\mu^{\beta}} \right) $.
$\left| Q_{r}^{*}-\mathcal{T}^{\pi} Q_{r}^{*} \right|_{\mu^{\beta}}$ is the additional sub-optimality error term, it comes from the fact that the optimal policy may not satisfy $\pi_{\beta}/\pi^{*} \geq \epsilon$. 
The filter $P_c$ allows the change of measure from $\nu$ to $\mu^{\beta}$ by bounding the concentration constant $C$, which captures the maximum density ratio between marginal distribution $\nu(s)$ and $\mu^{\beta}(s)$. Then the main theorem is proved by combining all those steps.
\end{proof}
\noindent \textbf{Summary} 
\quad We show that we can enlarge Qc-values of OOD actions to be greater than $l$ by adjusting $\alpha$ in Theorem 1. We also show the performance guarantee in Theorem 2. 
Note that we can vary $\epsilon$ to make $G(\epsilon)$ as small as possible by adjusting the parameter $d$ in Algorithm 1, which is the only hyperparameter that needs to be tuned.  This result guarantees that, upon the termination of Algorithm 1, the true performance of the main objective can be close to that of the optimal safe policy. At the same time, the safe constraint will be satisfied, assuming sufficiently large $k$.

\begin{figure*} [t]
\centering
\includegraphics[width=2.12\columnwidth]{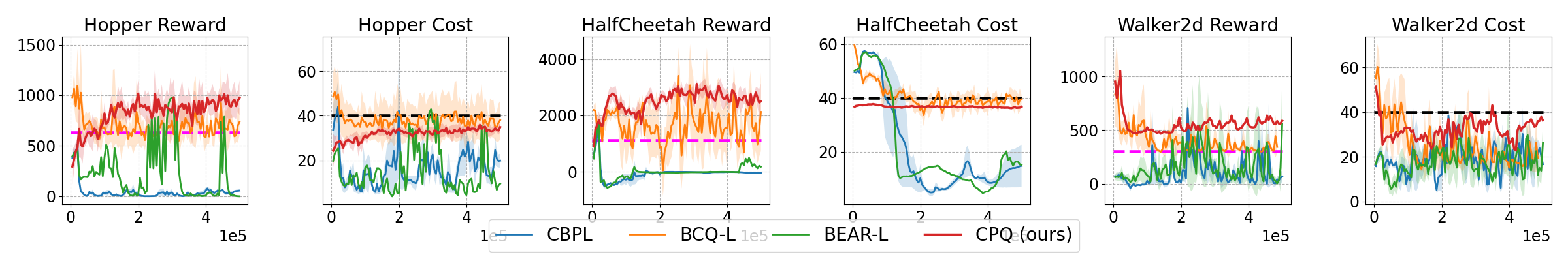}
\caption{We evaluate CPQ and different baselines according to the experiments of Section 6.1. 
The shaded area represents one standard deviation around the mean. 
The dashed magenta line measures the performance of BC-Safe. 
The constraint threshold $l$ is indicated by the dashed black line. 
It can be seen that CPQ is robust to learn from different scenarios, outperforms other baselines while still satisfying safe constraints.}
\label{fig1}
\end{figure*}

\section{Experiments}
\subsection{Settings}
We conducted experiments on three Mujoco tasks: \texttt{Hopper-v2}, \texttt{HalfCheetah-v2} and \texttt{Walker2d-v2}. These tasks imitate scenarios encountered by robots in real life. The robot is composed of multiple joints, at each step the agent selects the amount of torque to apply to each joint.
In the experiments, we aim to prolong the motor life of different robots, while still enabling them to perform tasks. To do so, the motors of robots need to be constrained from using high torque values. This is accomplished by defining the constraint $C$ as the discounted cumulative torque that the agent has applied to each joint, and per-state penalty $c(s, a)$ is the amount of torque the agent decides to apply at each step.

For each environment, we collect data using a \textit{safe policy} which has low rewards with safety constraints satisfied and an \textit{unsafe policy} which has high reward but violates safety constraints.
The unsafe policy was trained by PPO \cite{schulman2017proximal} until convergence to the returns mentioned in Figure 1 and the safe policy was trained by CPO \cite{achiam2017constrained} using the constraint threshold $l=30$.
The dataset is a mixture of 50\% transitions collected by the safe policy and 50\% collected by the unsafe policy. Mixture datasets are of particular interest, as it covers many practical use-cases where agents act safely in most cases but have some unsafe attempts for more benefits.
Each dataset contains 2e6 samples. We used the same dataset for evaluating different algorithms to maintain uniformity across results.

Each agent is trained for 0.5 million steps and evaluated on 10 evaluation episodes (which were separate from the train distribution) after every 5000 iterations, we use the average score and the variance for the plots.

\subsection{Baselines}
We compare CPQ with the following baselines:

\noindent \textbf{CBPL}: 
CBPL \cite{le2019batch}, which learns safe policies by applying FQE and FQI, was originally designed for discrete control problems, we extend it to continuous cases by using continuous FQI \cite{antos2008fitted}.

\noindent \textbf{BCQ-Lagrangian}: As BCQ \cite{fujimoto2019off} was not designed for safe offline RL, we combine BCQ with the Lagrangian approach, which uses adaptive penalty coefficients to enforce constraints, to obtain BCQ-Lagrangian.

\noindent \textbf{BEAR-Lagrangian}: Analogous to BCQ-Lagrangian, but to use another state-of-the-art offline RL method BEAR \cite{kumar2019stabilizing}.

\noindent \textbf{BC-Safe}: As mentioned in Section 3.2, we also include a Behavior Cloning baseline, using only data generated from the safe policy. This serves to measure whether each method actually performs effective RL, or simply copies the data.

\subsection{Comparative Evaluations}
It is shown in Figure 1 that CPQ achieves higher reward while still satisfying safety constraints, compared to two na\"ive approaches (BCQ-L and BEAR-L). 
Na\"ive approaches achieve sub-optimal performance due to the reasons discussed in Section 3.2. For example, BEAR-L
has difficulty to learn the balance between two Lagrangian multiplier, $\lambda_1$ for the safety constraint and $\lambda_2$ for the divergence constraint, this two multiplier raise their value by turns to try to satisfy either of the two constraints, making the effect of Qr diluted.
BCQ-L performs better than BEAR-L, but still suffers from zig-zag learning curves due to the similar reason.
CBPL diverges and fails to learn good policies in all three environments, due to large value estimation errors caused by OOD actions.
It can also be observed that the constraint values of CPQ are sometimes lower than the threshold, due to the reason that $\nu$ sometimes falsely chooses in-distribution actions, making the Qc-values of these actions erroneously large. This suggests that the performance of CPQ may be further enhanced by applying more advanced OOD detection techniques to construct $\nu$, we leave it for future work.

\subsection{Sensitivity to Constraint Limit $l$}

The results discussed in the previous section suggest that CPQ outperforms other baselines on several challenging tasks. 
We’re now interested in the sensitivity of CPQ to different constraint limit $l$.
It can be seen in Figure 2 that CPQ is robust to different constraint limits. This means that we can do counterfactual policy learning\footnote{In online RL under constraint, the agent needs to “re-sample-and-learn” from scratch when the constraint limit is modified.} \cite{garcia2015comprehensive}, i.e., adjust $l$ post-hoc to derive policies with different safety requirements. 
Note that imitation-based methods (e.g. BC-Safe) can only satisfy the original constraint limit $l$.

\begin{figure} [t]
\centering
\includegraphics[width=0.95\columnwidth]{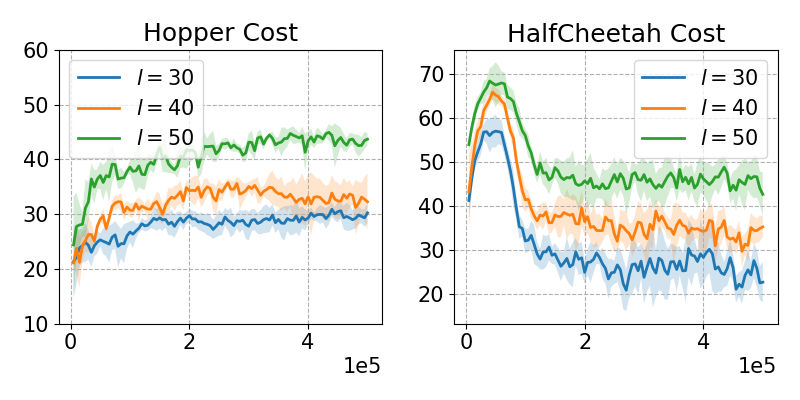}
\caption{Sensitivity to constraint limit $l$.}
\label{fig2}
\end{figure}




\section{Conclusions and Future Work}
We present a novel safe offline RL algorithm, CPQ, the first continuous control RL algorithm capable of learning from mixed offline data under constraints. 
Through theoretical analysis as well as systematic experimental results, we show that CPQ achieves better performance across a variety of tasks, comparing to several baselines.
One future work is to use more advanced OOD detection techniques (e.g., using energy scores \cite{liu2020energy}), to further enhance CPQ's performance. Another future work is developing new algorithms to tackle offline RL under hard constraints. 
We hope our work can shed light on safe offline RL, where one could train RL algorithms offline, and provides reliable policies for safe and high quality control in real-world tasks.

\section*{Acknowledgments}
A preliminary version of this work was accepted as a spotlight paper in RL4RealLife workshop at ICML 2021. This work was supported by the National Key R\&D Program of China (2019YFB2103201) and the National Science Foundation of China (No. 61502375).

\bibliography{aaai22}

\appendix
\onecolumn

\end{document}